\documentclass[10pt]{article}
\expandafter\let\csname equation*\endcsname\relax
\expandafter\let\csname endequation*\endcsname\relax
\usepackage{bm}
\usepackage{pifont}
\usepackage{enumerate}
\usepackage[top=1in, bottom=1in, left=1in, right=1in]{geometry}
\usepackage[dvipsnames]{xcolor}
\usepackage{mathrsfs}
\usepackage{amsfonts}
\usepackage{multirow}
\usepackage{upgreek}
\usepackage{nicefrac}  
\usepackage{amssymb,dsfont}
\usepackage{caption,comment}
\usepackage{blkarray}
\usepackage{amsmath}
\usepackage{float}
\usepackage{footnote}
\usepackage{amssymb,amsthm}
\usepackage[toc,page]{appendix}
\usepackage{graphicx,afterpage}
\usepackage{epstopdf}
\usepackage[normalem]{ulem}
\usepackage{algorithm}
\usepackage{algorithmic}
\usepackage{xspace}
\usepackage{enumitem}
\usepackage{authblk}
\usepackage{hyperref}
\usepackage{natbib}

\usepackage{enumitem}
\usepackage{nicefrac}

\def\real{\mathbb{R}}

\newcommand{\cP}{\mathcal P}

\newcommand{\KL}{\mathrm{KL}}

\def\F{\textrm{FR}}

\def\W{\textrm{W}}

\usepackage{tikz}
\usetikzlibrary{positioning}

\newtheorem{assumption}{Assumption}
\newtheorem{theorem}{Theorem}
\newtheorem{lemma}{Lemma}
\newtheorem{remark}{Remark}
\newtheorem{proposition}{Proposition}

\newtheorem{example}{Example}

\title{Preservation of Log-Concavity and Convergence of Wasserstein--Fisher--Rao Gradient Flows}
\date{}

\author[1]{Francesca Romana Crucinio\thanks{ \href{mailto:francescaromana.crucinio@unito.it}{francescaromana.crucinio@unito.it}
   }}
\author[2]{Sahani Pathiraja\thanks{  \href{mailto:s.pathiraja@unsw.edu.au}{s.pathiraja@unsw.edu.au} } }
\date{ }

\affil[1]{ESOMAS, University of Turin, Italy \& Collegio Carlo Alberto, Turin, Italy}
\affil[2]{School of Mathematics \& Statistics, UNSW Sydney, Australia}

\begin{document}

    \maketitle

\begin{abstract}

We study the convergence of Wasserstein--Fisher--Rao (WFR) gradient flows for sampling from probability distributions known up to a normalisation constant. By combining Wasserstein transport with Fisher--Rao birth--death dynamics, WFR flows balance exploration and selection.  These flows have been recognised as a promising mechanism to accelerate convergence beyond Langevin dynamics. We show that for a class of strongly log-concave target distributions satisfying additional curvature conditions, WFR flows preserve strong log-concavity, in contrast to Wasserstein flows which enjoy this property only in the Gaussian setting. Exploiting this result, we derive explicit non-asymptotic convergence rates for the symmetrised Kullback--Leibler divergence, without requiring a warm-start as required in current estimates. In particular, we show that the convergence rate decomposes additively into Wasserstein and Fisher--Rao contributions, thereby confirming a recent conjecture within this setting. These results provide refined convergence guarantees and further develop the theoretical foundations of WFR gradient flows for sampling and Bayesian inference.

\end{abstract}


\section{Introduction}

We consider the task of generating samples from a target probability distribution known up to a normalisation constant with density $\pi(x) \propto e^{-V_\pi(x)}, \enskip x \in \mathbb{R}^d$.  Despite the conceptual simplicity of this task, its efficient implementation when $V_\pi$ is multi-modal, the underlying space is high dimensional and/or modes are separated by large distances remains challenging.   Given the broad application of sampling to (Bayesian) statistics and statistical machine learning, several avenues have been investigated to derive efficient and scalable algorithms.

A natural way to formulate this task is via gradient flows, which can be seen as optimisation of a functional measuring the dissimilarity to $\pi$, typically the Kullback--Leibler (KL) divergence \citep{wibisono2018sampling, crucinio2025note, chen2023sampling}.
This formulation yields considerable freedom in the design of sampling algorithms; arguably the most well-known being the so-called Wasserstein Gradient flow, hereafter W flow \citep{jordan1998variational}.  
It is well-known that when $\pi$ satisfies a Log-Sobolev Inequality (LSI), the W flow converges exponentially fast to $\pi$, with rate depending on the Log-Sobolev constant.  This highlights an inherent limitation of W flows, namely that when the LSI constant is large (e.g. as is typically the case in multi-modal densities with well separated modes), the convergence rate can be prohibitively slow \citep{schlichting2019poincare}.

Recent research efforts have instead considered gradient flows in the Fisher--Rao geometry (FR flow).  These are well known in the biological literature as describing the macroscopic properties of a population with varying traits or species, and is sometimes referred to as birth-death or replicator dynamics \citep{kimura_stochastic_1965,schuster1983replicator, Cressman2006}. 

It is known that FR flows are intimately connected to mirror descent \citep{chopin2023connection}, stochastic filtering \citep{akyildiz2017probabilistic, halder_gradient_2017, Pathiraja2024, DelMoral1997} and sequential Monte Carlo \citep{us}.  They have also recently been exploited to develop sampling algorithms \citep{Nusken2024, maurais2024sampling, Wang2024, chen2023sampling, lu2023birth, Lu2019}, most notably due to the fact that it is possible to achieve convergence rates independent of the properties of $V_\pi$ \citep{carrillo_fisher-rao_2024, Lu2019}.

In this work, we consider the so-called Wasserstein--Fisher--Rao (WFR) gradient flow, wherein the metric is given by the direct sum of the Wasserstein and Fisher-Rao metrics.
We refer the reader to \citep{liero_optimal_2018} for a rigorous treatment of this metric and its corresponding gradient flow. 
WFR gradient flows combine the diffusive behaviour of W flows with the birth-death or reactive properties of FR flows and enjoy better convergence properties than both W and FR alone \citep{Lu2019, chen2023sampling}.  They have a natural interpretation as combining `exploration' or `mutation' to provide new particles (W flow) with `selection' where particles that are a poor fit to the target are killed \citep{Pathiraja2024}.  The convergence properties of WFR flows have been theorised to improve on both W and FR flows, yet the best known results are obtained under strong conditions on the ratio between $\pi$ and the initial distribution $\mu_0$ and require a warm-start condition (\citet[Appendix B]{Lu2019} and \citet[Remark 2.6]{lu2023birth}).

We consider convergence of the WFR gradient flow for strongly log-concave targets. Under the curvature assumptions stated below, we obtain precise decay rates for the symmetrised KL without a warm-start condition. Te key step is a strong log-concavity preservation result for the WFR flow. This is in stark contrast with the W flow, which is known to preserve log-concavity uniformly in time only in the Gaussian case \citep{Kolesnikov2001}.

Our main contributions are as follows:
\begin{itemize}
    \item We establish conditions under which the WFR preserves strong log-concavity (Section~\ref{sec:logconcave}). This result combines a finite time horizon result on log-concavity preservation for the W flow with the strong regularising properties of the FR flow.
    \item We obtain a non-asymptotic convergence result for the continuous time WFR flow which shows that the rate of convergence is the sum of the rate of the W flow and that of the FR flow (Section~\ref{sec:convlogconc}), as conjectured in \cite{domingo-enrich2023an}. 
\end{itemize}

\textbf{Notation}
We define some notation that will be used throughout the manuscript.
For all differentiable functions $f$ we denote the gradient by $\nabla f$. Furthermore, if $f$ is twice differentiable we denote by $\nabla^2f$ its Hessian and by $\Delta f$ its Laplacian. 
 We denote by $\cP(\real^d)$ the set of probability measures over
$\mathcal{B}(\real^d)$, and endow this space with the topology of weak convergence. 
We denote by $\cP_2^{ac}(\real^d)$ the manifold of absolutely continuous probability measures on $\real^d$ with finite second moment. Every $p\in \cP_2^{ac}(\real^d)$ will be identified with its (Lebesgue) density $p(x)$.  Throughout this manuscript we denote the target by $\pi$ and assume $\pi(x) \propto e^{-V_\pi(x)}$. The initial distribution is denoted by $\mu_0$.  
The Kullback--Leibler divergence is defined for $\nu,\mu$ admitting a density w.r.t. Lebesgue as $\KL(\nu||\mu)=\int \log(\nicefrac{\nu(x)}{\mu(x)}) \nu(x)dx$.

\section{Wasserstein--Fisher--Rao Gradient Flow}

The gradient flow of $\KL(\mu || \pi)$ w.r.t. the geometry induced by the Wasserstein--Fisher--Rao distance is given by the following PDE \citep[Theorem 3.1]{Lu2019}
\begin{align}
    \label{eq:WFRpde}
    \partial_t \mu_t &= f_\W(\mu_t)+f_{\F}(\mu_t),\\
    f_\W(\mu) &:= \nabla \cdot (\mu \nabla \log \frac{\mu}{\pi}), \label{eq:winfflow}\\
    f_{\F}(\mu) &:= -\mu \left(\log \frac{\mu}{\pi} -  \mathbb{E}_\mu \left[ \log \frac{\mu}{\pi} \right]  \right).\label{eq:infFR}
\end{align}
where $f_\W$ and $f_{\F}$ are the Wasserstein and Fisher--Rao operators respectively. 
This PDE is a combination of the Wasserstein (W) flow and of the Fisher--Rao (FR) flow.  While the W flow cannot be solved analytically, the FR flow admits a closed form solution given by (e.g. \citet[App. B.1]{chen2023sampling}, \citet[Eq. (16)]{lu2023birth})
\begin{align}
\label{eq:fr_semigroup}
   \mu_t^{\textrm{FR}}(x)\propto \pi(x)^{1 - e^{-t}} \mu_0(x)^{e^{-t}}.
\end{align}

The rate of decay of $\KL$ for the WFR flow is obtained considering the time derivative of $\KL(\mu_t||\pi) $ along the flow via classical arguments valid under log-concavity assumptions, 
\begin{align}
\label{eq:derivative_kl}
\frac{d}{dt} \KL(\mu_t||\pi) &=\int \log \frac{\mu_t}{\pi} \partial_t  \mu_t= - \int \mu_t\vert\nabla\log \frac{\mu_t}{\pi}\vert^2 -\text{Var}_{\mu_t} \left[ \log  \frac{\mu_t}{\pi}  \right].
\end{align}
As observed in \citet[page 13]{gallouet2017jko}, the first term corresponds to the negative gradient of $\KL$ w.r.t. the Wasserstein-2 metric, while the remaining terms give the negative gradient of $\KL$ w.r.t. the Fisher--Rao geometry, implying that the dissipation for the WFR metric is the sum of the W and the FR dissipation.
The instantaneous WFR dissipation is the sum of the W and FR dissipations and is therefore at least as large as either component dissipation
$$\KL(\mu_t||\pi)\leq \min\left\lbrace \KL(\mu_t^{\textrm{FR}}||\pi), \KL(\mu_t^{\textrm{W}}||\pi)\right\rbrace,$$
where $\mu_t^{\textrm{FR}}, \mu_t^{\textrm{W}}$ denote the solutions of the FR and W flow PDEs respectively.

Convergence of the W flow is guaranteed under a Log-Sobolev assumption
\begin{align}
    \label{eq:lsi}
    \KL(\mu||\pi)\leq \frac{\lambda_\pi}{2}\int \mu \left\lvert\nabla\log \frac{\mu}{\pi}\right\rvert^2:= \frac{\lambda_\pi}{2}\mathcal{I}(\mu||\pi),
\end{align}
where $\lambda_\pi<\infty$ denotes the Log-Sobolev constant of $\pi$ and $\mathcal{I}(\mu||\pi) $ is the relative Fisher information between $\mu$ and $\pi$. This assumption also guarantees that the W flow is well-defined.
Under~\eqref{eq:lsi}, we can obtain a bound on the decay of KL along the W flow (e.g., \cite{chewi2024analysis})
\begin{align*}
    \KL(\mu_t^{\textrm{W}}||\pi) \leq e^{-2\lambda_\pi^{-1}t}\KL(\mu_0||\pi).
\end{align*}

To the best of our knowledge, the convergence result in KL obtained under the weakest assumptions on $\pi, \mu_0$ is given in \citet[Eq. (B.6)]{chen2023sampling} and requires that $\mu_0, \pi$ have bounded second moments $\int |x|^2\mu_0(x)dx \leq B$, $\int |x|^2\pi(x)dx \leq B$
 and $\left|\log\nicefrac{\mu_0(x)}{\pi(x)}\right|\leq M(1+|x|^2)$
for some $M>0$. Under these assumptions, we have
\begin{align*}
\KL(\mu_t^{\textrm{FR}}||\pi)\leq Me^{-t}(2+B+Be^{M e^{-t}(1+B)}).
\end{align*}

The resulting WFR rate is
\begin{align}
\label{eq:rate_wfr}
\KL(\mu_t||\pi)\leq \min\left\lbrace e^{-2\lambda_\pi^{-1}t}\KL(\mu_0||\pi), Me^{-t}(2+B+Be^{M e^{-t}(1+B)})\right\rbrace.
\end{align}
This rate is never worse than that of W or FR, looking at~\eqref{eq:derivative_kl}, however it is evident that the rate of decay of the WFR flow should be considerably higher than that of the two singular flows, and in particular it should be the sum of the singular rates as conjectured in \cite{domingo-enrich2023an}.

A sharper rate of converge for WFR can be obtained assuming $\log\nicefrac{\pi(x)}{\mu_0(x)}\geq M$ and a warm-start condition $\KL(\mu_0|\pi)\leq 1$ \citep[Appendix B]{Lu2019}
\begin{align}
\label{eq:sharp_wfr}
  \KL(\mu_t||\pi) &\leq  e^{-\left(2\lambda_\pi^{-1}+(2-3\delta)\right)(t-t_0)}\KL(\mu_0||\pi),
\end{align}
for all $t\geq t_0:=\log(M/\delta^3)$ and $\delta>0$.
This result shows that the decay in the case of the WFR flow is faster than that of the W flow but requires a warm-start condition, moreover the rate is only valid for $t\geq t_0:=\log(M/\delta^3)$ which grows as $\delta\to 0$ making~\eqref{eq:sharp_wfr} sharp only for large $t$.

As a check of the sharpness of the above rates we consider a 1D Gaussian target for which the exact KL decay is available as shown in \cite{us_splitting} (Figure~\ref{fig:rate}). Even in this simple case, the rates available in the literature become sharp only for large $t$.
\begin{figure}
	\centering
	\begin{tikzpicture}[every node/.append style={font=\normalsize}]
\node[node distance = 0, xshift = -1cm] {\includegraphics[width = 0.5\textwidth]{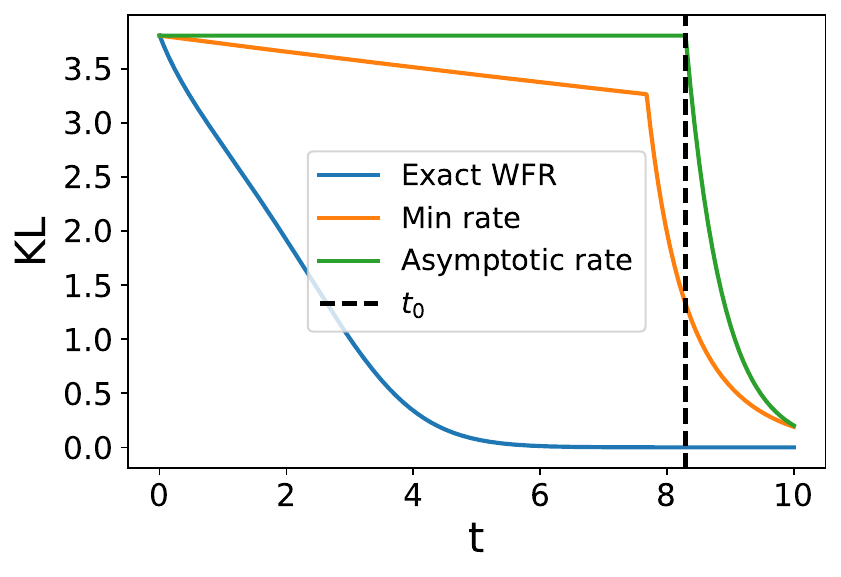}};
	\end{tikzpicture}
\caption{Comparison of exact KL decay (blue) compared to rates in the literature,~\eqref{eq:rate_wfr} (orange) and~\eqref{eq:sharp_wfr} (green) with $\delta = 0.1, t_0 = 8.3$ for a 1D Gaussian with $m_\pi = 20, C_\pi  = 100, m_0 = 0, C_0 = 1$.}

\label{fig:rate}
\end{figure}

\section{Preservation of log-concavity}
\label{sec:logconcave}

As a first result towards establishing convergence of the WFR flow we obtain conditions under which the W, FR and WFR preserve log-concavity when initialised from a log-concave initial distribution $\mu_0$. 
Recall that for any log-concave $\mu_0$, preservation of log-concavity uniformly in time under the W flow is guaranteed only for $\pi$ Gaussian (i.e. for Ornstein-Uhlenbeck Semigroups) due to \cite{Kolesnikov2001}. This is in stark contrast to the FR flow, where log-concavity is preserved uniformly in time for any strongly log-concave $\pi$ and $\mu_0$ (see Lemma \ref{lem:FRlogconc}).

With some stronger conditions on $\mu_0$ and $\pi$, as detailed in Assumptions \ref{ass:logconcave} and \ref{ass:WLC}, we establish preservation of log-concavity under the W flow for a limited time horizon, using similar arguments as in \citet[Section 7.2]{Pathiraja2021} (in the context of stochastic filtering) and Lemma 16 in \cite{liang_characterizing_2025}. We first state the necessary assumptions:

\begin{assumption}
\label{ass:logconcave} The following conditions hold  
\begin{enumerate}[label=(\alph*)]
\item $\pi(x) \propto e^{-V_\pi(x)}$, with $V_\pi$ continuously differentiable, is $\alpha_\pi$-strongly log-concave and $L_\pi$-smooth, i.e. there exists an $L_\pi\geq \alpha_\pi > 0$ such that $ L_\pi I \succeq \nabla^2 V_\pi(x) \succeq \alpha_\pi I$ for all $x \in \mathbb{R}^d$;
\item $\mu_0(x) \propto e^{-V_0(x)}$, with $V_0$ continuously differentiable, is $\alpha_0$-strongly log-concave and $L_0$-smooth, i.e. there exists an $L_0\geq \alpha_0 > 0$ such that $ L_0 I \succeq \nabla^2 V_0(x) \succeq \alpha_0 I$ for all $x \in \mathbb{R}^d$.
\end{enumerate}
\end{assumption}

The assumption on $\pi$ assumes that $\pi$ is log-concave and implies that $\nabla V_\pi$ is Lipschitz continuous with Lipschitz constant $L_\pi$. While strong, this assumption is classical in the study of the W flow.
The assumption on $\mu_0$ is mild since $\mu_0$ is often user-chosen and guarantees that $\nabla V_0$ is Lipschitz continuous with constant $L_0$. 

\begin{assumption}
\label{ass:WLC} Given Assumption~\ref{ass:logconcave}, the following conditions hold.  
\begin{enumerate}[label=(\alph*)]
\item \label{ass:WLC1} $V_0 - \frac{(1 + \delta)}{2}V_\pi$ is strongly convex with parameter $\alpha_d > 0$ for some specified $0 < \delta < 1$.   
\item Denote by $R:= -\frac{1}{2}\Delta V_\pi + \frac{1}{4}|\nabla V_\pi|^2$. Also define $\mathcal{H}:= R + V_\pi$.  Assume $\mathcal{H}$  is strongly convex with parameter $\alpha_h > 0$. 
\end{enumerate}
\end{assumption}
The first condition requires the initial density to be sufficiently strongly log-concave relative to the target density. As $\mu_0$ is often user chosen, this assumption is not overly restrictive. The second condition is a uniform coercivity condition on the Schr\"odinger potential $\mathcal H$; this potential appears in the path-integral representation of the W flow and controls the curvature contribution accumulated along the path. These assumptions are sufficient conditions for the result and are not claimed to be necessary.

The next lemma gives the precise statement of the time horizon over which the W flow is guaranteed to preserve strong log-concavity. 
Our proof strategy involves starting with a sequential splitting of the WFR flow and using Girsanov theorem to characterise the intermediate density due to the W flow. 
This is based on a similar proof strategy used to show that the filtering density satisfies a Poincar\'e inequality (see Lemma 5.1 in \cite{Pathiraja2021}). The proof is given in Appendix  \ref{sec:proofWlogconc}
\begin{lemma}[W flows preserve log-concavity]
    \label{lem:Wlogconc} Suppose $\mu_t$ is the solution of the W flow \eqref{eq:winfflow} at time $t$, initialised at $\mu_0(x)$ satisfying Assumption \ref{ass:WLC}.  
Let $b:= \sqrt{\frac{|\alpha_h - L_\pi|}{2}}$. Then for some  $0 < \delta < 1$ as in Assumption \ref{ass:WLC}, the following holds 
\begin{enumerate}[label=(\roman*)]
    \item \textbf{Fixed time horizon}:  Suppose $\alpha_h - L_\pi < 0$.  Then there exists a $t^\ast > 0$ such that for all $t < t^\ast$, $ \mu_t(x) \propto e^{-\mathcal{E}_t(x)}$, with $\mathcal{E}_t(x)$ a strongly convex function where $\nabla^2 \mathcal{E}_t(x) \succeq \left( \frac{\alpha_\pi}{2} + c_t \right) I$ for all $x \in \mathbb{R}^d, $ where
    \begin{align}
    \label{eq:ctWflow}
        c_t &=  b \tan \left(\tan^{-1}\left( \frac{c_0}{b} \right) - 2b t  \right) \\
         \label{eq:coeqn}
    c_0 &= \alpha_d + \frac{\delta}{2} \alpha_\pi,
    \end{align}
    \item \textbf{Uniform in time}: Suppose $\alpha_h - L_\pi > 0$.  Then for all $t \geq 0$, $ \mu_t(x) \propto e^{-\mathcal{E}_t(x)}$, with $\mathcal{E}_t(x)$ a strongly convex function where $\nabla^2 \mathcal{E}_t(x) \succeq \left( \frac{\alpha_\pi}{2} + c_t \right) I$ for all $x \in \mathbb{R}^d, $ where 
    \begin{align}
    \label{eq:ctWflowuniform}
        c_t = b \left( \frac{K - e^{-4bt}}{K +  e^{-4bt}} \right) = b \left( \frac{b + c_0 - e^{-4bt}(b - c_0)}{b + c_0 + e^{-4bt}(b - c_0)} \right) 
    \end{align}
    with $K  := \frac{b + c_0}{b - c_0}  $ and $c_0$ as in \eqref{eq:coeqn}. 
\end{enumerate} 

\end{lemma}
When $\alpha_h - L_\pi < 0$, a uniform in time preservation of log-concavity is not guaranteed, and the length of the time horizon over which strong convexity is preserved depends on the log-concavity of the difference between initial potential $V_0$ and (some factor) of the target potential $V_\pi$. Loosely speaking, the time horizon increases as $\alpha_d$ in Assumption~\ref{ass:WLC} increases.

\begin{remark}[Gaussian distributions]
Condition (ii) in Lemma~\ref{lem:Wlogconc} is not necessary in the Gaussian case (but is sufficient beyond the Gaussian case for uniform in time preservation). In fact, in the Gaussian case $$R(x) = -\frac{1}{2}\Delta V_\pi(x) + \frac{1}{4}\left|\nabla V_\pi(x)\right|^2 = -\frac{1}{2}Tr[C_\pi^{-1}] + \frac{1}{4}(x-m_\pi)^\top C_\pi^{-2}(x-m_\pi)$$
    and $\nabla^2 R(x) = \frac{1}{2}C_\pi^{-2}\succeq 0$. It follows that when checking the strong convexity of $\nabla^2_{w_1}f_2$ in the proof of Lemma~\ref{lem:Wlogconc}, no condition on $\alpha_h, L_\pi$ is needed.  However, this is unique to the Gaussian case and does not apply in general. 
\end{remark}

Although the W flow cannot be generally expected to preserve log-concavity uniformly in time (unless $\pi$ is Gaussian), the properties of the FR flow can be exploited to maintain log-concavity uniformly in time. The next lemma shows that the FR preserves log-concavity when $\mu_0, \pi$ satisfy Assumption~\ref{ass:logconcave}.
\begin{lemma}[FR flows preserve log-concavity]
    \label{lem:FRlogconc}
 Suppose $\mu_t$ is the solution of the FR flow \eqref{eq:infFR} at time $t$, initialised at $\mu_0(x)$ satisfying Assumption \ref{ass:logconcave}. Then $\mu_t, \enskip t > 0$ is  $\alpha_t$-strongly log-concave with
    \begin{align*}
        \alpha_t = (1-e^{-t})\alpha_\pi + e^{-t}\alpha_0
    \end{align*}
    \end{lemma}
    \begin{proof}
    Using the exact solution to the FR flow \eqref{eq:fr_semigroup},
     \begin{align*}
            \mu_t(x) \propto \pi(x)^{1 - e^{-t}} \mu_0 (x) ^{e^{-t}} \propto e^{-V_t(x)}          
        \end{align*}
        where $ V_t = (1 - e^{-t})V_\pi + e^{-t}V_0$.  Also 
        \begin{align*}
            \nabla^2 V_t =  (1 - e^{-t}) \nabla^2 V_\pi + e^{-t} \nabla^2 V_0  \succeq ((1-e^{-t})\alpha_\pi + e^{-t}\alpha_0)I =: \alpha_t I
        \end{align*}
        where $\alpha_t > 0$ is a convex combination of positive scalars for all $t > 0$.  
    \end{proof}

In the next theorem, we show that the preservation of log-concavity under the FR flow can be exploited to ensure the same for the WFR flow uniformly in time, even when the W flow may not preserve log-concavity uniformly.  The proof can be found in Appendix \ref{app:logconcuni}.    

\begin{theorem}[WFR preserves strong log-concavity uniformly in time]
\label{theo:logconc}
         Assume the conditions of Lemma \ref{lem:Wlogconc}.  For the case $\alpha_h - L_\pi < 0$ only, additionally assume that $b:= \sqrt{\frac{|\alpha_h - L_\pi|}{2}}$ satisfies  
    \begin{align}
        \label{ass:b2}
        b^2 < \tfrac{\alpha_\pi}{4}. 
    \end{align} 
    Then $\mu_t$, the solution of \eqref{eq:WFRpde} at time $t$ is $\alpha_t$-strongly log-concave for all $t > 0$, 
    where  
\begin{align}
\label{eq:logconcuniform}
    \alpha_t 
    & = \frac{\alpha_\pi}{2} -\frac{1}{4} + m \cdot \frac{(m+c_0+\frac14)-(m-c_0-\frac14)e^{-4mt}}{(m+c_0+\frac14)+(m-c_0-\frac14)e^{-4mt}}
\end{align}
where $c_0$ is given by \eqref{eq:coeqn}, $\enskip m = \sqrt{\frac{1}{2} \left( \frac{1}{8} + r \right)}$ and 
\begin{align*}
    r  &= \begin{cases}
             \frac{\alpha_\pi}{2} - 2b^2, \quad \alpha_h - L_\pi < 0 \\
             \vspace{-0.4cm} \\
             \frac{\alpha_\pi}{2} + 2b^2, \quad \alpha_h - L_\pi > 0
         \end{cases}
\end{align*}
Note also that $ \alpha_\infty  > \frac{\alpha_\pi}{2}$.
\end{theorem}

In the case $\alpha_h-L_\pi<0$, the W flow does not preserve log-concavity uniformly in time and we require condition~\eqref{ass:b2}. This condition is sufficient to ensure that there exists a lower bound to the possible curvature deficit of the Schr\"odinger potential. 
In the case $\alpha_h-L_\pi>0$, the W flow preserves log-concavity uniformly, and combining this with Lemma~\ref{lem:FRlogconc} via the same splitting argument as in the proof of Theorem \ref{theo:logconc} (without restrictions on $b$ as in \eqref{ass:b2}) yields the result.  

A comparison between Theorem~\ref{theo:logconc} and the exact log-concavity constant from \cite{us_splitting, liero2025evolution} in the Gaussian case shows that the constant in Theorem~\ref{theo:logconc} is tight (Figure~\ref{fig:alpha_t}).
We further evaluate the tightness of the constant \eqref{eq:logconcuniform} on a 1D example.

\begin{example}[1D Non-Gaussian target]
\label{ex:nongauss}
We consider a non-Gaussian target, corresponding to a perturbation of a Gaussian: for $\beta>0, c_\pi>0$ set
\begin{align}
    \label{eq:target_ex}
    V_\pi(x) &= \frac{x^2}{2 c_\pi} + \beta \log(1 + e^x).
\end{align}
We first check that $\pi$ satisfies our assumptions: we have
\begin{align*}
    V_\pi''(x) &= \frac{1}{c_\pi} + \beta(1+e^{-x})^{-2} e^{-x} = \frac{1}{c_\pi} + \beta u(x) (1 - u(x))
\end{align*}
where $u(x) = (1 + e^{-x})^{-1}$ and $0 < u(x) < 1$ for all $ x \in \mathbb{R}$.
Clearly
\begin{align*}
    V_\pi''(x) &\geq  \frac{1}{c_\pi} + 0 = \alpha_\pi 
\end{align*}
for all $x\in\real$.
Moreover $V_\pi''(x)$ has a unique maximum at $x=0$ and thus
\begin{align*}
    V_\pi''(x) \leq \frac{1}{c_\pi} + \frac{\beta}{4} =: L_\pi.
\end{align*}
This shows that Assumption~\ref{ass:logconcave} is satisfied.


To check Assumption~\ref{ass:WLC} consider $V_0 = \frac{x^2}{2\sigma^2_0}$. Then,
\begin{align*}
      V_0'' - \frac{(1+\delta)}{2} V_\pi'' &= \frac{1}{\sigma^2_0} - \frac{(1+\delta)}{2} \left( \frac{1}{c_\pi} + \beta u(x) (1 - u(x)) \right) \\
      & \geq \frac{1}{\sigma^2_0} - \frac{(1+\delta)}{2}  \frac{1}{c_\pi} - \frac{(1+\delta)}{2} \frac{\beta}{4} =: \alpha_d 
\end{align*}
To guarantee $\alpha_d>0$ is sufficient to take an initial distribution such that 
$\frac{1}{c_\pi}
+\frac{\beta}{4}
<
\frac{2}{(1+\delta)\sigma^2_0}$.
Since $\delta\in (0,1)$, the factor $(1+\delta)/2$ is increasing in $\delta$ and is bounded above by $1$. Thus $\alpha_d>0$
provided that
\begin{align}
\label{eq:c0_condition}
\frac{1}{c_\pi}
+\frac{\beta}{4}
<
\frac{1}{\sigma^2_0}.
\end{align}
For the second condition, take
\begin{align*}
 \mathcal{H}(x) &= -\frac{1}{2}V_\pi''(x) + \frac{1}{4}\left( V_\pi'(x)\right)^2 + V_\pi(x),
\end{align*}
whose second derivative is
\begin{align*}
\mathcal{H}''(x) 
=&-\frac{\beta}{2}u(x)(1-u(x))(1-6u(x)+6u(x)^2)\\
&+\frac{1}{2}\left(\frac{1}{c_\pi}+\beta u(x)(1-u(x))\right)^2\\
&+\frac{\beta}{2}\left(\frac{x}{c_\pi}+\beta u(x)\right)u(x)(1-u(x))(1-2u(x))\\
&+\frac{1}{c_\pi}+\beta u(x)(1-u(x)).
\end{align*}
As $0<u(x)< 1$, for $\beta>0$ we have that the second term in the expression above is lower bounded by $1/(2c_\pi^2)$ and the last term is lower bounded by $1/c_\pi$. Moreover, we have $0<u(x)(1-u(x))<1/4$ and $|1-6u(x)+6u(x)^2|\leq 1$ and thus the first term is lower bounded by $-\beta/8$.
The remaining term can be written as $T(x)=T_1(x)+T_2(x)$ where
\begin{align*}
    T_1(x)&= \frac{\beta}{2c_\pi}
x u(x)(1-u(x))(1-2u(x))\\
T_2(x) &=\beta^2
u^2(x)(1-u(x))(1-2u(x)).
\end{align*}
For the second term, since $|1-2u(x)|\le1$
and $u^2(x)(1-u(x))\le u(x)(1-u(x))\le\frac14$,
we have the lower bound $-\frac{\beta^2}{4}$. 
For the first term, using known trigonometric identities we have
\[
T_1(x)
=
-\frac{\beta}{2c_\pi}\frac{x\tanh(x/2)}{4\cosh^2(x/2)}.
\]
Since \(x\) and \(\tanh(x/2)\) have the same sign, we can conclude that $T_1(x)\leq 0$. 
Moreover, using the fact that $|\tanh(x/2)|\leq 1$ and $\cosh(x/2)\ge \frac12 e^{|x|/2}$ we find
\begin{align*}
    |T_1(x)|\leq  \frac{\beta}{2c_\pi}|x|e^{-|x|}\leq \frac{\beta}{2ec_\pi}
\end{align*}
Therefore, we get the lower bound
$T_1(x)
\ge
-\frac{\beta}{2ec_\pi}$.

Combining all lower bounds gives
\begin{align}
    \label{eq:h_condition}
    \mathcal{H}''(x) \geq -\frac{\beta}{8}+\frac{1}{2c_\pi^2}-\frac{\beta}{2ec_\pi}
-\frac{\beta^2}{4}+\frac{1}{c_\pi}:=\alpha_h.
\end{align}

In Figure \ref{fig:alpha_t}(right) we consider the case $\sigma^2_0 = 1, c_\pi = 2, \beta = 1/2$ for which~\eqref{eq:c0_condition} trivially holds and~\eqref{eq:h_condition} is positive. 
The condition \eqref{ass:b2}  corresponds to 
\begin{align*}
   \frac{1}{2c_\pi}+\frac{\beta}{4} <\alpha_h  < \frac{1}{c_\pi}+\frac{\beta}{4}
\end{align*}
which is also satisfied. We compare the results of Theorem~\ref{theo:logconc} with the log-concavity constant obtained by numerically solving the WFR PDE with target $\pi$.
In this case \eqref{eq:logconcuniform} is reasonably tight, although its quality degrades for large $t$ and is clearly dependent on the estimate for $\alpha_h$. Although $\alpha_\infty$ can be different from $\alpha_\pi$, the true asymptotic constant, it is never worse than $\frac{\alpha_\pi}{2}$.  
\end{example}

We will rely on the results of Theorem~\ref{theo:logconc} in the next section to characterise convergence rates. 

\begin{figure}
    \centering
    \includegraphics[width=0.32\textwidth]{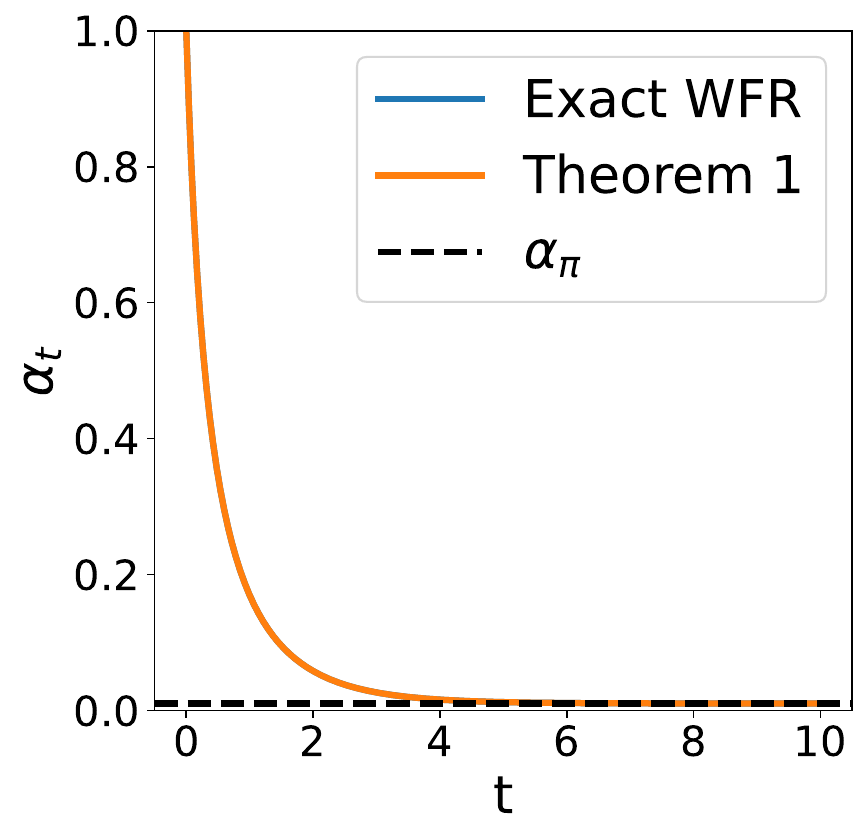}
    \hfill
    \includegraphics[width=0.32\textwidth]{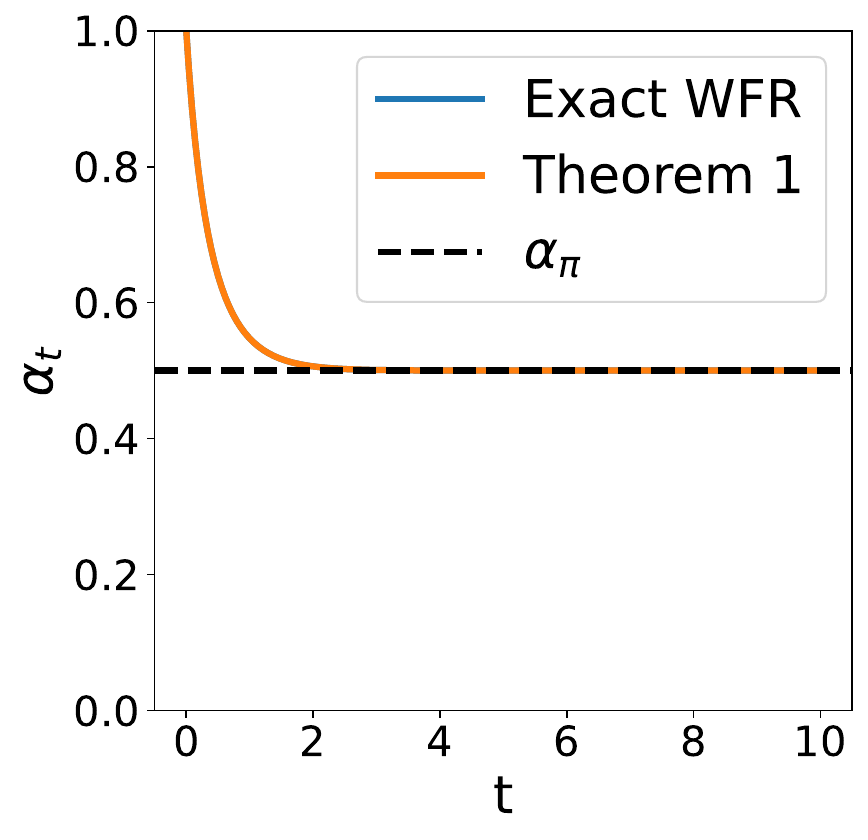}
    \hfill
    \includegraphics[width=0.32\textwidth]{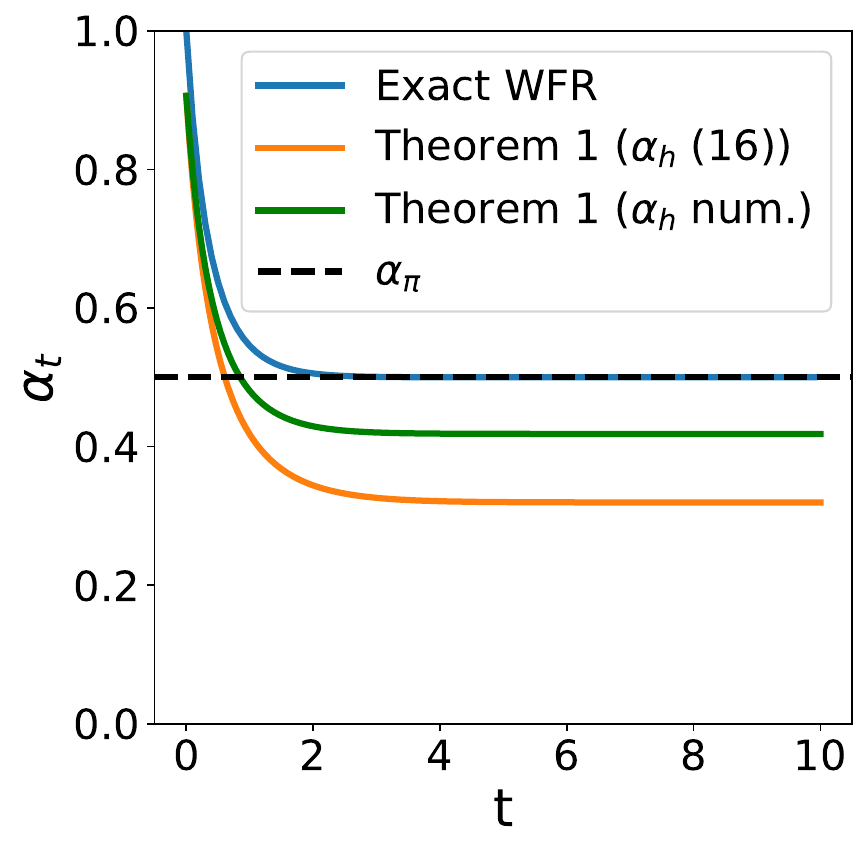}
    \caption{Comparison of log-concavity constant from Theorem~\ref{theo:logconc} and true log-concavity constant for the solution of \eqref{eq:WFRpde} for 1D targets.  Target densities considered are $\pi(x) = \mathcal{N}(x; 0, 100)$ (left), $\pi(x) = \mathcal{N}(x; 0, 2)$ (middle) and $\pi(x) \propto \exp \left(  -\frac{x^2}{4} - 0.5\log(1 + e^x)\right)$ (right).  For the non-Gaussian target, we consider $\alpha_t$ from Theorem 1 using $\alpha_h$ developed via the analytic calculations in \eqref{eq:h_condition} (orange line) as well as $\alpha_h$ calculated numerically (green line).   For the Gaussian targets, $\alpha_t$ is known analytically  whereas for the non-Gaussian target, $\alpha_t$ is evaluated numerically from an Euler discretisation of \eqref{eq:WFRpde}. In all cases, $\mu_0(x) = \mathcal{N}(x; 0, 1)$. In the Gaussian case, the log-concavity constants estimated by \eqref{eq:logconcuniform} are tight. For the third target, \eqref{eq:logconcuniform} is reasonably tight, although its quality degrades for large $t$ and is dependent on the quality of $\alpha_h$ (compare green vs orange line). }
    \label{fig:alpha_t}
\end{figure}

\section{Convergence of Wasserstein--Fisher--Rao gradient flow}
\label{sec:convlogconc}

The preservation of log-concavity under the WFR flow established in the previous section can be exploited to obtain explicit rates of convergence to $\pi$ for the WFR flow without a warm start condition as~\eqref{eq:sharp_wfr} or bounded moment conditions~\eqref{eq:rate_wfr}. In Proposition \ref{prop:decayJexactWFR}, we obtain an additive W and FR dissipation rate for the symmetrised KL rather than for the KL alone.  
The symmetrised $\KL$ divergence, or Jeffreys' divergence $J$, is defined as
\begin{align*}
    J(\mu, \pi):= \KL(\mu||\pi) + \KL(\pi||\mu),
\end{align*}
for any $\mu \ll \pi$ and also $\pi \ll \mu$.  The rationale behind considering the symmetrised KL is due to the fact that the KL is not geodesically convex under the FR flow \citep[Theorem 1.1]{carrillo_fisher-rao_2024} nor satisfies a gradient dominance condition \citep[Theorem 4.1]{carrillo_fisher-rao_2024}, which makes its convergence analysis more difficult.
On the other hand, Jeffreys' divergence satisfies a gradient dominance condition \citep[Section 4]{carrillo_fisher-rao_2024} which allows to achieve improved rates. 
In addition, $J(\mu, \pi)$ upper bounds $\KL(\mu||\pi)$ and thus the obtained rates for $J(\mu, \pi)$ gives information on the decay of KL too. The decay rate obtained in Proposition \ref{prop:decayJexactWFR} is a sum of the decay rates of W and FR flows\footnote{where the constant for the W flow is instead written in terms of the LSI constant, since W flows preserve LSI uniformly in time \citet[Lemma 16]{liang_characterizing_2025}.} as was conjectured by \cite{domingo-enrich2023an}.


\begin{proposition}
\label{prop:decayJexactWFR} 
    Assume the conditions of Theorem~\ref{theo:logconc}.
    Then the following decay result holds for $\mu_t$ the solution of the WFR PDE \eqref{eq:WFRpde} for all $t > 0$
    \begin{align*}
        J(\mu_t, \pi) \leq J(\mu_0, \pi)e^{-t(\alpha_\pi +1 )}. 
    \end{align*}
\end{proposition}
 \begin{proof}
    Given the assumptions, we may interchange differentiation and integration to obtain 
    \begin{align}
    \nonumber 
    \frac{d}{dt} J(\mu_t, \pi) &=\int \left( \log \frac{\mu_t}{\pi} -\frac{\pi}{\mu_t} \right)(x) \partial_t  \mu_t(x) dx \\
    \label{eq:Jderiv}
    & = \int \left( \log \frac{\mu_t}{\pi} -\frac{\pi}{\mu_t} \right)(x) \left(f_\W(\mu_t(x)) + f_\F(\mu_t(x))\right)  dx.
\end{align}
Begin with the first term, by a standard integration by parts argument, 
\begin{align}
    \int \left( \log \frac{\mu_t}{\pi} -\frac{\pi}{\mu_t} \right)(x) f_\W(\mu_t(x)) dx & = \int \left( \log \frac{\mu_t(x)}{\pi(x)} -\frac{\pi(x)}{\mu_t(x)} \right) \nabla \cdot \left( \mu_t(x) \nabla \log \frac{\mu_t}{\pi}(x) \right) (x) dx  \notag\\   
    \notag
    & = -\mathbb{E}_{\mu_t} \left[\left| \nabla  \log \frac{\mu_t}{\pi}  \right|^2 \right]  - \mathbb{E}_{\pi} \left[  \nabla \log \frac{\mu_t}{\pi} \cdot \nabla \log \frac{\mu_t}{\pi} \right] \\
    \notag
    & = -(\mathcal{I}(\mu_t|| \pi) + \mathcal{I}(\pi||\mu_t)) \\
    \label{eq:LSIineq}
    & \leq -2\min(\lambda_{\pi}^{-1}, \lambda_{\mu_t}^{-1})(\KL(\mu_t|| \pi) + \KL(\pi||\mu_t)) \\
    \notag
    & =  -2\min(\alpha_{\pi}, \alpha_{t})J(\mu_t, \pi), \\
    \label{eq:alphapiineq}
    & < -\alpha_\pi J(\mu_t,\pi),
\end{align}
where $\mathcal{I}(\mu_t|| \pi)$ is the relative Fisher information between $\mu_t$ and $\pi$. The inequality \eqref{eq:LSIineq} holds due to Assumption \ref{ass:logconcave} and Theorem \ref{theo:logconc}.  More specifically, since $\pi$ is $\alpha_\pi$-strongly log-concave, it immediately satisfies a Log-Sobolev inequality \eqref{eq:lsi} with constant $\lambda_{\pi} = \alpha_\pi^{-1}$ (likewise for $\mu_t$, which is $\alpha_t$-strongly log-concave due to Theorem \ref{theo:logconc}).   
The final inequality \eqref{eq:alphapiineq} holds because Theorem~\ref{theo:logconc} gives $\alpha_t>\frac{\alpha_\pi}{2}$ for every finite $t\geq0$. 

Then for the second term in \eqref{eq:Jderiv},  by a direct application of (5.16) in \citet[Theorem 5.6]{carrillo_fisher-rao_2024} with $f(y) = y \log(y)$ and $\bar{f} = yf(y^{-1})$, it holds that 
\begin{align*}
     -\int \frac{\pi(x)}{\mu_t(x)} f_\F(\mu_t(x)) dx = -J(\mu_t, \pi), 
\end{align*}
 and also, 
\begin{align*}
    \int \log \frac{\mu_t(x)}{\pi(x)}  f_\F(\mu_t(x)) dx = -\text{Var}_{\mu_t} \left( \log \frac{\mu_t}{\pi} \right)  < 0,
\end{align*}
so that together, we have
\begin{align}
\label{eq:proof_fr}
    \int \left( \log \frac{\mu_t}{\pi} -\frac{\pi}{\mu_t} \right)(x) f_\F(\mu_t(x)) dx < - J(\mu_t, \pi). 
\end{align}
An application of Gr\"{o}nwall's lemma then yields the result. 
    \end{proof}

\begin{figure}
	\centering
	\includegraphics[width = 0.45\textwidth]{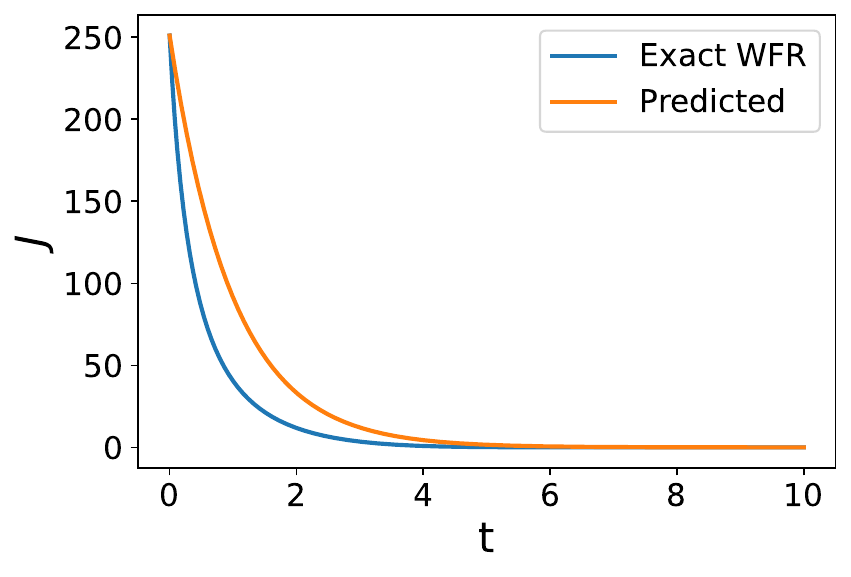}
    \includegraphics[width = 0.45\textwidth]{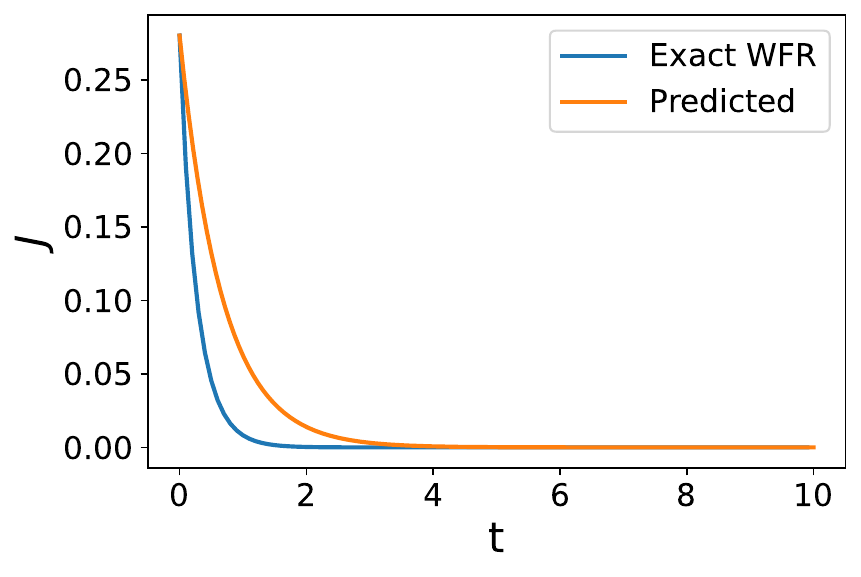}
\caption{Comparison between the exact symmetrised KL decay from \cite{us_splitting} (blue) with our rate in Proposition~\ref{prop:decayJexactWFR} (orange) for (left) a 1D Gaussian with $m_\pi = 20, C_\pi  = 100, m_0 = 0, C_0 = 1$ and (right) the target in Example~\ref{ex:nongauss}.}

\label{fig:rate_j}
\end{figure}

Proposition~\ref{prop:decayJexactWFR} shows that the rate of convergence of the WFR flow has the additive W and FR form conjectured by \cite{domingo-enrich2023an}, up to the factor introduced by the strong log-concavity estimate. This result is obtained for the strongly log-concave case without requiring a warm-start condition. The analysis uses the symmetrised KL instead of the more classical reverse KL.
In Figure~\ref{fig:rate_j} we compare the rate in Proposition~\ref{prop:decayJexactWFR} with the exact decay rate obtained in \cite{us_splitting} for a 1D Gaussian target, showing that the decay predicted by our result is sharp.

We point out that as $\KL(\mu_t||\pi)\leq J(\mu_t, \pi)$, Proposition~\ref{prop:decayJexactWFR} implies
\begin{align}
\label{eq:kl_decay_symmetrised}
    \KL(\mu_t||\pi)\leq J(\mu_0, \pi)e^{-t(\alpha_\pi +1 )}
\end{align}
showing that the same exponential rate is inherited by the KL. However, the constant $J(\mu_0, \pi)$ can be significantly larger than the constants in~\eqref{eq:rate_wfr} making the bound~\eqref{eq:kl_decay_symmetrised} looser for small $t$.

\begin{remark}
    A sharper convergence rate can be obtained by applying Gr\"{o}nwall's lemma before~\eqref{eq:alphapiineq}:
    \begin{align*}
        J(\mu_t, \pi) \leq J(\mu_0, \pi)\exp\left(-\left(2\int_0^t\min(\alpha_{\pi}, \alpha_{u})du + t \right) \right). 
    \end{align*}
    This bound clearly shows that the decay rate for the WFR flow is the sum of that for the W flow and that for the FR flow.
    However, as $\alpha_u$ is generally not known this bound provides less information than the one in Proposition~\ref{prop:decayJexactWFR}.
\end{remark}

\section{Discussion}

In this work, we established new theoretical guarantees for Wasserstein--Fisher--Rao (WFR) gradient flows in the setting of strongly log-concave target distributions. Under the explicit curvature assumptions in Section~\ref{sec:logconcave}, our main result proves that WFR dynamics preserve strong log-concavity uniformly in time. While Wasserstein flows preserve log-concavity only under restrictive conditions, most notably in the Gaussian setting, the addition of the Fisher--Rao component provides a regularising mechanism that maintains strong log-concavity uniformly in time. This structural result enables a refined analysis of the dynamics and forms the foundation for our convergence theory.

Building on this preservation property, we derived explicit non-asymptotic convergence guarantees for the continuous-time WFR flow without requiring warm-start assumptions. In particular, we proved exponential convergence of the symmetrised Kullback--Leibler divergence with an additive Wasserstein--Fisher--Rao dissipation rate, thereby confirming the conjecture of \cite{domingo-enrich2023an}. The same estimate yields a bound on the KL through $\KL(\mu_t\|\pi)\leq J(\mu_t,\pi)$. The analysis also highlights how the transport and birth–death components contribute additively to the overall dissipation, providing a precise mathematical characterisation of the complementary roles of exploration and selection in WFR dynamics.

Several aspects of our analysis could likely be strengthened. The assumptions ensuring uniform preservation of log-concavity are sufficient but may not be necessary, and it would be of considerable interest to determine whether they can be relaxed or replaced by weaker geometric conditions. Likewise, while our convergence result focuses on strongly log-concave targets, extending these techniques to broader classes of distributions—including weakly log-concave, non-convex, or multimodal targets—would significantly increase the applicability of the theory. Since WFR dynamics have been proposed precisely to improve sampling performance in challenging landscapes, understanding the extent to which these theoretical guarantees persist beyond the log-concave regime remains an important open question.

More generally, our results contribute to the growing understanding of hybrid optimal transport geometries by illustrating how transport and reaction mechanisms interact to produce stronger qualitative and quantitative properties than either component alone. Beyond the specific setting considered here, it would be interesting to investigate whether similar structural preservation results hold for other gradient flows defined on combined geometries or for alternative energy functionals arising in sampling, inference, and optimisation. We hope that the techniques developed in this work provide a foundation for further study of Wasserstein--Fisher--Rao dynamics and their role in the analysis of gradient flows on spaces of probability measures.

Finally, beyond the specific convergence result obtained here, the proof strategy itself may be of independent interest. By exploiting the complementary properties of the Wasserstein and Fisher--Rao operators through a splitting argument, we are able to establish qualitative properties of the combined flow that are difficult to obtain directly from the reaction-diffusion equation. This perspective may prove useful in the analysis of other hybrid gradient flows and provides a natural link with operator-splitting formulations of WFR dynamics such as those considered in \cite{us_splitting}.

\paragraph*{Acknowledgments}

F.R.C. gratefully acknowledges the ``de Castro" Statistics Initiative at the \textit{Collegio Carlo Alberto} and the \textit{Fondazione Franca e Diego de Castro}. F.R.C. is supported by the Gruppo
Nazionale per l'Analisi Matematica, la Probabilità e le loro Applicazioni (GNAMPA-INdAM).
S.P. gratefully acknowledges funding from UNSW Faculty of Science Research Grant and the Eva Mayr Stihl Foundation.

The authors would like to thank Josh Bon, Sam Power, Florian Maire, Daniel Paulin, Upanshu Sharma for helpful discussions and in particular Sinho Chewi for pointing to the reference \cite{Kolesnikov2001} and Andre Wibisono for directing us to Lemma 16 in \cite{liang_characterizing_2025}.  The authors gratefully acknowledge the mathematical research institute MATRIX in Australia where part of this research was performed.

\appendix

\section{Preservation of log-concavity}
\label{app:logconc}

\subsection{Proof of Lemma \ref{lem:Wlogconc}}
\label{sec:proofWlogconc} 

The first step is to use Girsanov's theorem to characterise the law of the overdamped Langevin diffusion
\begin{align}
\label{eq:overdamped}
    dY_t = -\nabla V_\pi(Y_t)\,dt + \sqrt{2}\,d\widetilde{W}_t,
\end{align}
where $\{\widetilde{W}_t\}_{t\ge0}$ is a standard Brownian motion. We instead consider the time-rescaled process $X_s = Y_{s/2}$ obtained by the time change $s = 2t$.
By Brownian scaling, $\widetilde{W}_{s/2} = 2^{-1/2}W_s$, where $\{W_s\}_{s\ge0}$ is again a standard Brownian motion. Hence, $X_s$ satisfies
\begin{align}
    \label{eq:rescaledlangevin}
    dX_s = \nabla U(X_s)ds  + dW_s.
\end{align}
with $U = -\frac{1}{2} V_\pi$. 
Note that \eqref{eq:rescaledlangevin} has the same invariant density as \eqref{eq:overdamped}. 

Let $\widehat{\mu}_s$ denote the law of $X_s$, $\mu_W(dw)$ denote the Wiener measure on the space of continuous paths on $[0, s]$, $C([0,s]; \mathbb{R}^d)$ with $W_0 \sim \mu_0$
and $\varphi: \mathbb{R}^d \rightarrow \mathbb{R}$ be a bounded measurable test function.  By an application of Girsanov's theorem (see \citet[Exercise 8.15]{oksendal2003stochastic}),
\begin{align*}
    \int \varphi(x) \widehat{\mu}_s(dx) &= \int_{C([0,s]; \mathbb{R}^d)} \varphi(w_s)\exp \left(  \int_0^s \nabla U (W_u)\cdot dW_u - \frac{1}{2} \int_0^s |\nabla U (W_u)|^2 du \right)\mu_W(dw) \\
    & = \int_{(C[0,s]; \mathbb{R}^d)} \varphi(w_s) \exp \left(  (U(W_s) - U(W_0) - \frac{1}{2}\int_0^s (\Delta U + |\nabla U|^2)(W_u)du   \right)\mu_W(dw) 
\end{align*}
Note that Girsanov's theorem applies here since Novikov's condition is easily verified as $U$ satisfies a linear growth condition $|\nabla U| \leq c_1 (1 + |x|)$ for some $c_1 > 0$ under Assumption~\ref{ass:logconcave}.  In the second line, we have used It\^o's formula to obtain 
\begin{align*}
    U(W_s) - U(W_0) = \int_0^s \frac{1}{2}\Delta U (W_u) du + \int_0^s \nabla U (W_u) \cdot dW_u.
\end{align*}

Consider now a discretisation of the time interval $[0,s]$ with step size $\tau$ and $N$ steps, $0 = s_0, s_1, s_2, \dots, s_{N-1}, s_N = s$ with $s_i - s_{i-1} = \tau$ for all $i = 1, 2, 3, \dots, N$.  From now on, we use the notation $\widehat{\mu}_i$ to denote an approximation to $\widehat{\mu}_{s_i}$ (and likewise for other quantities). Recalling that 
\begin{align*}
    R(w) &= -\frac{1}{2}\Delta V_\pi(w) + \frac{1}{4}|\nabla V_\pi(w)|^2 = \Delta U(w)+|\nabla U(w)|^2
\end{align*}
we can approximate
\begin{align*}
    \int_0^s (\Delta U + |\nabla U|^2)(W_u)du = \int_0^s R(W_u)du  \approx \tau\sum_{i=1}^{N-1}R(w_i)
\end{align*}
and we arrive at the approximation (with a slight abuse of notation on $\mu_W$),
\begin{align}
\label{eq:approx}
    \int \varphi(x) \widehat{\mu}_s(dx)  \approx &\int \varphi(w_N)\exp\left(\frac{1}{2}(V_\pi(w_0) - V_\pi(w_N)) \right) \prod_{i=1}^{N-1} \exp \left( -\frac{\tau}{2} R(w_{i}) \right) \mu_W(d w_0,\ldots,d w_N).
\end{align}
The Wiener measure factorises as
\begin{align*}
\mu_W(d w_0,\ldots,d w_N)
=
\exp(-V_0(w_0))
\prod_{i=1}^{N}
q_\tau(w_{i-1},w_i)\,
dw_0\cdots dw_N,
\end{align*}
with
\begin{align*}
     q_{\tau}(w_{{i-1}}, w_{i}) &:= \frac{1}{(2\pi \tau)^{d/2}} \exp \left(- \frac{|w_{i} - w_{{i-1}}|^2}{2 \tau}  \right).
\end{align*}
Plugging the above into~\eqref{eq:approx} we obtain
\begin{align*}
   &\int \varphi(x) \widehat{\mu}_s (dx) \approx \\
   & \int \varphi(w_N)  \exp\left(\frac{1}{2}(V_\pi(w_0) - V_\pi(w_N)) \right) \prod_{i=1}^{N-1}
Q(w_i) \exp(-V_0(w_0))
\prod_{i=1}^{N}
q_\tau(w_{i-1},w_i)\,
dw_0\cdots dw_N
\end{align*}
with 
\begin{align*}
     Q(w) := \exp \left( -\frac{1}{2} R(w)\tau \right). 
\end{align*}
That is, $\hat{\mu}_s$ is approximated by (up to normalising constants) by a measure with density 
\begin{align*}
   \exp\left( -\frac{1}{2}V_\pi(w_{N})\right) \mu_N(w_N) 
\end{align*}
and also
\begin{align*}
     \mu_i(w_i) := \left\{\begin{array}{lr}
        \int Q(w_{{i-1}}) \, q_{\tau}(w_{{i-1}}, w_{i}) \, \mu_{i-1}(w_{i-1}) \, dw_{{i-1}}, &  i = 2, 3, \dots \\
         \int Q(w_{0}) \, q_{\tau}(w_{0}, w_{i}) \, \exp\left(\frac{1}{2}V_\pi(w_0)\right) \, \exp(-V_0(w_0)) \, dw_0, & i = 1.
        \end{array} \right.
\end{align*}
Then by Assumption \ref{ass:WLC}, for $i = 1$, we have, 
\begin{align*}
    \mu_1 &\propto \int \exp \left(- \frac{|w_0 - w_1|^2}{2\tau} + \frac{1}{2} V_\pi(w_0)  -\frac{\tau}{2}R(w_0) - V_0(w_0) \right) dw_0 \\
    &= \int \exp \left(- \frac{|w_0 - w_1|^2}{2\tau} - \left(V_0(w_0) -  \frac{1 + \delta }{2} V_\pi(w_0) \right) - \frac{\delta - \tau }{2} V_\pi(w_0)  -\frac{\tau}{2} \left(R(w_0) + V_\pi(w_0) \right)   \right) dw_0 \\
    & =: \int \exp(-f_1(w_1, w_0)) dw_0\\
    & = \exp(-\mathcal{G}_1(w_1)),
\end{align*}
assuming $\tau < \delta$.  Since $f_1$ is jointly strongly convex in $(w_0, w_1)$ under Assumption \ref{ass:WLC},  by the Pr\'ekopa--Leindler inequality \citep[Theorem 4.3, page 380]{MR450480}, so is $\mathcal{G}_1$ with 
\begin{align}
\label{eq:prekopleindler}
    \nabla_y^2 \mathcal{G}_1 \succeq \frac{\int  (\nabla_y^2 f_1 - \nabla_{yz}^2 f_1 (\nabla_z^2 f_1)^{-1} \nabla_{zy}^2 f_1  ) \exp(-f_1(y,z))dz}{\int \exp(-f_1(y,z))dz},
\end{align}
 using the shorthand notation $y:=w_1$ and $z:= w_0$. 
Then,
\begin{align*}
   \nabla_{w_{1}}^2 f_1 &= \frac{1}{\tau}I, \quad   
   \nabla_{w_{1}w_{0}}^2 f_1 = \nabla_{w_{0}w_{1}}^2 f_1 = -\frac{1}{\tau}I, \\
   \nabla_{w_{0}}^2 f_1 & = \nabla_{w_0}^2 (V_0 - \frac{1+\delta }{2} V_\pi) + \frac{\tau}{2} \nabla_{w_0}^2 \mathcal{H} + \frac{\delta - \tau}{2} \nabla_{w_0}^2 V_\pi    + \frac{1}{\tau}I \\
   & \succeq \left( \alpha_d + \frac{\tau}{2} \alpha_h + \frac{\delta - \tau}{2} \alpha_\pi + \frac{1}{\tau} \right) I
\end{align*}
So then 
\begin{align}
\nonumber 
     \nabla_y^2 \mathcal{G}_1 &\succeq \frac{\int  (\frac{1}{\tau} - \frac{1}{\tau^2} (\nabla_{w_0}^2 f_1)^{-1}   ) \exp(-f_1(y,z))dz}{\int \exp(-f_1(y,z))dz} \\
     \nonumber 
     & \succeq  \left(\frac{1}{\tau} - \frac{1}{\tau^2} \left(   \alpha_d + \frac{\tau}{2} \alpha_h + \frac{\delta - \tau}{2} \alpha_\pi + \frac{1}{\tau}  \right)^{-1} \right) I \frac{\int  \exp(-f_1(y,z))dz}{\int \exp(-f_1(y,z))dz} \\
    \label{eq:c1formlanew}
     & = \frac{\alpha_d + \frac{\delta}{2} \alpha_\pi + \frac{\tau}{2}(\alpha_h - \alpha_\pi)}{1 + \tau (\alpha_d + \frac{\delta}{2} \alpha_\pi + \frac{\tau}{2}(\alpha_h - \alpha_\pi))} I \\
     \nonumber 
     &=: c_1 I 
\end{align}
which is strictly positive by all the assumptions and since $\alpha_d + \frac{\delta}{2} \alpha_\pi + \frac{\tau}{2}(\alpha_h - \alpha_\pi) = \alpha_d + \frac{\tau}{2}\alpha_h + \frac{\delta - \tau}{2}\alpha_\pi > 0$ whenever $\delta > \tau$.  Then for $i=2$, 
\begin{align*}
    \mu_2 &\propto \int \exp \left( -\frac{| w_1 - w_2|^2}{2 \tau} -\frac{\tau}{2}R(w_1) - \mathcal{G}_1(w_1) \right) dw_1 \\
    & = \int \exp \left( -\frac{| w_1 - w_2|^2}{2 \tau} -\frac{\tau}{2}(R(w_1) + V_\pi(w_1)) + \frac{\tau}{2}V_\pi(w_1) - \mathcal{G}_1(w_1) \right) dw_1 \\
    &=:\int \exp \left( -f_2(w_1, w_2)\right) dw_1.
\end{align*}
Once again, $f_2$ is jointly strongly convex in $(w_1, w_2)$ and also 
\begin{align*}
    \nabla_{w_1}^2 f_2 &= \frac{1}{\tau}I + \frac{\tau}{2}\nabla_{w_1}^2 \mathcal{H}(w_1) -\frac{\tau}{2} \nabla_{w_1}^2 V_\pi(w_1) + \nabla_{w_1}^2 \mathcal{G}_1(w_1) \\
    & \succeq \left( \frac{1}{\tau} + \frac{\tau}{2}\alpha_h + c_1 - \frac{\tau}{2}L_\pi\right)I.
\end{align*}
A sufficient condition to maintain convexity is to choose $\tau$ small enough such that $c_1 - \frac{\tau}{2}L_\pi > 0$ (which is possible since $c_1 \rightarrow \alpha_d + \frac{\delta}{2} \alpha_\pi > 0$ as $\tau \rightarrow 0$).  Then again by Pr\'ekopa--Leindler, $\mu_2 \propto \int \exp(-f_2(w_1, w_2))dw_1 = \exp(-\mathcal{G}_2(w_2))$ with 
\begin{align*}
    \nabla^2_{w_2} \mathcal{G}_2 &\succeq \left( \frac{1}{\tau} - \frac{1}{\tau^2} \left( \frac{1}{\tau} + \frac{\tau}{2}\alpha_h + \kappa_1(\tau) \right)^{-1} \right)I  \\
    & = \frac{\kappa_1(\tau) + \frac{\tau}{2}\alpha_h}{1 + \tau(\kappa_1(\tau) + \frac{\tau}{2}\alpha_h)} I =: c_2 I
\end{align*}
where $\kappa_1(\tau) := c_1 - \frac{\tau}{2}L_\pi $.  Repeating for $\mu_3$, we have exactly the same computations, but with the requirement that $\tau$ is chosen small enough that $c_2 - \frac{\tau}{2}L_\pi > 0$, which by similar arguments as previously, holds true for some sufficiently small $\tau$.  By induction, we conclude that for all $i = 1, 2, 3, \dots, N$,  $\mu_i \propto \exp(-\mathcal{G}_i(w_i))$ is strongly log-concave with $\nabla_{w_i}^2 \mathcal{G}_i \succeq c_i I$ where 
\begin{align}
    \label{eq:ciiteration}
    c_i &= \frac{c_{i-1}+ \frac{\tau}{2}(\alpha_h - L_\pi)}{1 + \tau (c_{i-1}+ \frac{\tau}{2}(\alpha_h - L_\pi))},  \quad i = 1, 2, 3, \dots, N \\
    c_0 &= \alpha_d + \frac{\delta}{2} \alpha_\pi.
\end{align} 
Notice that \eqref{eq:ciiteration} holds for $i=1$ since Since $x\mapsto x/(1+\tau x)$ is increasing on $x>-1/\tau$ and $\alpha_\pi<L_\pi$; replacing $\alpha_\pi$ by $L_\pi$ in \eqref{eq:c1formlanew} only weakens the lower bound.

Finally, returning to our original approximation, we have that 
\begin{align*}
    \widehat{\mu}_t \approx & \exp \left(-\frac{1}{2} V_\pi (w_{N})  \right)\mu_N(w_N)   \propto \exp \left( -\frac{1}{2} V_\pi (w_{N}) - \mathcal{G}_N \right) 
\end{align*}
and 
\begin{align}
\label{eq:finalcn}
     \frac{1}{2} \nabla^2V_\pi (w_{N}) + \nabla^2\mathcal{G}_N(w_N)  \succeq \left(\frac{1}{2}\alpha_\pi  + c_N \right) I.
\end{align}
In order to understand the limit $\tau \rightarrow 0$ of the recursion \eqref{eq:ciiteration}, notice that it can be seen as a two step time discretisation, i.e.  
\begin{align}
    \label{eq:splitciter1}
    \tilde{c}_i &= c_{i-1} \pm \tau b^2 \\
     \label{eq:splitciter2}
    c_i &= \frac{\tilde{c}_i}{1 + \tau \tilde{c}_i}
\end{align}
with $b^2 = \frac{|\alpha_h - L_\pi|}{2}$ and the plus or minus in \eqref{eq:splitciter1} depending on the sign of $\alpha_h - L_\pi$.  Also, \eqref{eq:splitciter1} corresponds to an Euler discretisation of the ODE $\frac{dc_s}{ds} = \pm b^2$ with time step $\tau$.  Notice also that \eqref{eq:splitciter2} is the exact solution of $\frac{dc_s}{ds} = -c_s^2$ initialised at $\tilde{c}_i$ at time $s = \tau$ since if $c_s = c_0(1+ s c_0)^{-1}$ then differentiating with respect to $s$ yields $\dot{c}_s = -c_0^2(1 + s c_0)^{-2} = -c_s^2$.  Therefore, the iteration \eqref{eq:splitciter1}-\eqref{eq:splitciter2} corresponds to an operator splitting based discretisation of the ODE 
\begin{align}
    \label{eq:logconcWode}
    \frac{dc_s}{ds} = -c_s^2 \pm  b^2.
\end{align}
Notice that for the case $\frac{dc_s}{ds} = -c_s^2 -  b^2$, the ODE has no stable fixed point, and continues decreasing to $-\infty$ as $t \rightarrow \infty$. 
For this case, we focus on characterising the time horizon over which strong convexity is preserved.  
Integrating both sides of \eqref{eq:logconcWode} yields 
\begin{align}
    \label{eq:ctau}
    c_s &= b \tan \left(\tan^{-1} \left( \frac{c_0}{b} \right) - b s  \right), \\ 
    c_0 &=  \alpha_d + \frac{\delta}{2} \alpha_\pi 
\end{align}
where $b := \sqrt{\frac{|\alpha_h - L_\pi|}{2}}$, noting this excludes the final $\frac{1}{2}\alpha_\pi$ term. 
Finally, we have that
\begin{align*}
    \widehat{\mu}_t &\propto \exp \left( -\frac{1}{2} V_\pi - \mathcal{G}_{t}\right)  =: \exp \left( - \mathcal{E}_t \right) \\
    \nabla^2 \mathcal{E}_t & \succeq \left(c_{t}(c_0) + \frac{1}{2}\alpha_\pi \right) I.
\end{align*}
Now there exists some $s^\ast$ beyond which strong convexity of the intermediate densities is lost, that is, $c_s \leq 0$ for all $s \geq s^\ast$.  Due to the monotonocity of the tan function, $s^\ast$ is easily obtained as 
\begin{align}
\label{eq:tauthreshc0}
     s^\ast &= \frac{1}{b}\tan^{-1}\left( \frac{c_0}{b} \right) - \frac{1}{b}\tan^{-1} \left(0 \right)  = \frac{1}{b}\tan^{-1}\left( \frac{c_0}{b} \right) 
\end{align}
which is strictly positive since $b, c_0 > 0$.  A change of time-scale back to $s = 2t$ to obtain the concavity constants for \eqref{eq:overdamped} yields the result \eqref{eq:ctWflow} and the time until this holds is given by $t^\ast = \frac{1}{2b} \tan^{-1} \left( \frac{c_0}{b}\right)$.   

Now consider the case when $\alpha_h - L_\pi > 0$, for which the corresponding ODE of interest is given by 
\begin{align*}
    \frac{dc_s}{ds} = -c_s^2 + b^2. 
\end{align*}
A simple separation of variables calculation yields 
\begin{align*}
    c_s &= b \left( \frac{K- e^{-2bs} }{K  + e^{-2bs} } \right), \qquad
    K  = \frac{b + c_0}{b - c_0}.
\end{align*}
Once again, applying the time-rescaling $s = 2t$ yields \eqref{eq:ctWflowuniform} and $c_t > 0$ for all $t$.

\subsection{Proof of Theorem \ref{theo:logconc}}
\label{app:logconcuni}

To establish this result we consider a splitting scheme for the WFR flow \citep{us_splitting}.
Consider a time-discretisation $t_0 = 0 < t_1 < t_2 \cdots < t_M = T$ with $t_{i+1} - t_i = \gamma \enskip \forall \enskip i = 1, 2, \dots M$.  A sequential splitting applied to \eqref{eq:WFRpde} takes the form 
\begin{align}
\label{eq:sequential_split}
    \hat{\mu}_i(x; \gamma) &= S_\W(\gamma,  \mu_{i-1})(x) \rightarrow \mu_i(x; \gamma) = S_\F(\gamma,\hat{\mu}_{i})(x), \quad i = 1,2, 3, \dots  
\end{align}
with initial distribution $\mu_0$ and
where $S_\F(\gamma, v)$ denotes the solution operator corresponding to $f_{\F}$ in \eqref{eq:fr_semigroup} acting on $v$ over time interval of size $\gamma$ and likewise $S_\W(\gamma, v)$ denotes the solution operator corresponding to $f_\W$.  We refer to~\eqref{eq:sequential_split} as Wasserstein then Fisher--Rao scheme, or W-FR for short.  We note that the step size $\gamma$ for the W and FR components.   
It is known that $\mu_n$ converges to the true WFR solution $\mu_t$ at rate $\mathcal{O} (\gamma)$ under mild conditions on $f_\F$, $f_\W$ and $\pi$ (e.g. \cite{Hundsdorfer2003}). 

We first consider the case $\alpha_h - L_\pi < 0$.  We first show that strong log-concavity is preserved for a sequence of steps with step size $\{\tau_i\}_{i=1, 2, \dots}$, and then show that strong log-concavity is preserved as $\tau_i \rightarrow 0$, over an infinite time horizon.  As will become clear in the remainder of the proof, such an iteration dependent step size is needed due to the fact that the time horizon over which the W flow preserves strong log-concavity depends on the relative convexity of the initial and target potentials.

For a given strongly log-concave $\mu_0$ satisfying Assumptions \ref{ass:WLC}, we have that after a single step of the W flow of size $\tau_1$, the distribution $\widehat{\mu}_1$ is $\widehat{\alpha}_1$-strongly log-concave due to Lemma \ref{lem:Wlogconc} with potential $\mathcal{E}_{1} = \frac{1}{2}V_\pi + \mathcal{G}_{1}$, 
where $\nabla^2 \mathcal{E}_{1} \succeq \widehat{\alpha}_{1}I$ and $\nabla^2 \mathcal{G}_{1} \succeq \widehat{c}_{1}I$, $\widehat{\alpha}_1 = \widehat{c}_1 + \frac{\alpha_\pi}{2}$ and  
\begin{align*}
    \widehat{c}_{1} &= b \tan \left(\tan^{-1}\left( \frac{c_0}{b} \right) - 2b \tau_1  \right) \\
    c_0 &= \alpha_d + \frac{\delta_0}{2} \alpha_\pi,
\end{align*}
where $0 < \delta_0 < 1$ and $\tau_1$ must be ``small enough'' that $\widehat{c}_1 > 0$, and due to Lemma \ref{lem:Wlogconc} and the assumed conditions, such a $\tau_1 > 0$ exists.  For the FR step, it trivially holds that 
\begin{align*}
    \mu_1 &\propto \pi^{1 - e^{-\tau_1}}\widehat{\mu}_1^{e^{-\tau_1}}  \propto \exp \left(-(1-e^{-\tau_1})V_\pi -e^{-\tau_1}\mathcal{E}_1  \right)  =: \exp(-\mathcal{F}_1)
\end{align*}
where  $\mathcal{F}_1 = V_\pi + e^{-\tau_1}(\mathcal{E}_1 - V_\pi) = V_\pi + e^{-\tau_1}(\mathcal{G}_{1} - \frac{1}{2}V_\pi)$ yields that the distribution from a single sequential split step of size $\tau_1$ is strongly log-concave with constant 
\begin{align*}
    \alpha_1 = \alpha_\pi + e^{-\tau_1}(\widehat{\alpha}_1 - \alpha_\pi) = \left(1 - \frac{1}{2}e^{-\tau_1} \right)\alpha_\pi + e^{-\tau_1} \widehat{c}_1.  
\end{align*}
For $i=2$,  we require that $\mathcal{F}_1 - \frac{1+\delta_1}{2}V_\pi $ is strongly convex to apply Lemma \ref{lem:Wlogconc}, for some $0 < \delta_1 < 1$.  Note that we use a different $\delta_1$ as compared to $\delta_0$ in the first step, as this will be used to define the sequence of step sizes.  Then
\begin{align*}
   \nabla^2 (\mathcal{F}_1 - \tfrac{1+\delta_1}{2}V_\pi) &= \nabla^2 \left( \left( 1 - \tfrac{1+\delta_1}{2} \right) V_\pi + e^{-\tau_1}(\mathcal{G}_{1} - \tfrac{1}{2}V_\pi) \right),  \\
   &= \nabla^2 \left( \left( \tfrac{1-\delta_1 - e^{-\tau_1}}{2} \right) V_\pi + e^{-\tau_1}\mathcal{G}_{1} \right) \\
   & \succeq \left(\left( \tfrac{1-\delta_1 - e^{-\tau_1}}{2} \right) \alpha_\pi + e^{-\tau_1} \widehat{c}_{1} \right) I =: \alpha_{d,1} I 
\end{align*}
where $\alpha_{d,1}$ denotes the concavity parameter required for Assumption~\ref{ass:WLC}\ref{ass:WLC1} in Lemma \ref{lem:Wlogconc}.  A sufficient condition for strict positivity of $\alpha_{d,1}$ is to ensure $0< \delta_1 < 1 - e^{-\tau_1}$. Then $\widehat{\mu}_2$ is again strongly convex by Lemma \ref{lem:Wlogconc} with potential $\mathcal{E}_2 = \frac{1}{2}V_\pi + \mathcal{G}_{2}$, $\nabla^2 \mathcal{G}_{2} \succeq \widehat{c}_{2}I$ and 
\begin{align*}
      \widehat{c}_{2} &= b \tan \left(\tan^{-1}\left( \tfrac{c_1}{b} \right) - 2b \tau_2  \right) \\
    c_1 &=  \left( \tfrac{1-\delta_1 - e^{-\tau_1}}{2} \right) \alpha_\pi + e^{-\tau_1} \widehat{c}_{1}+ \tfrac{\delta_1}{2} \alpha_\pi \\
    & = \tfrac{1}{2}\alpha_\pi + e^{-\tau_1} (\widehat{c}_{1} -\tfrac{\alpha_\pi}{2}) 
\end{align*}
and $c_1 > 0$ since $\tau_1$ is chosen such that $\widehat{c}_{1} > 0$.  Once again, $\tau_2$ must be chosen such that $\widehat{c}_{2} > 0$.  Then once again for the FR step, we have $ \mu_2 \propto \exp (-\mathcal{F}_2)$ with 
\begin{align*}
    \mathcal{F}_2 &= (1-e^{-\tau_2})V_\pi + e^{-\tau_2}\mathcal{E}_2 \\
    & = V_\pi + e^{-\tau_2}(\mathcal{G}_{2} - \frac{1}{2}V_\pi)
\end{align*}
and $\nabla^2 \mathcal{F}_2 \succ \alpha_2 I$ where $ \alpha_2 = \left(1 - \frac{1}{2}e^{-\tau_2} \right)\alpha_\pi + e^{-\tau_2}  \widehat{c}_2$.
Proceeding inductively for $i = 3, 4, \dots $, in a similar way with required conditions on $\tau_i$ and $\delta_i$, we obtain 
\begin{align}
\label{eq:alphafinalrecurs}
    \alpha_i &= \left(1 - \tfrac{1}{2}e^{-\tau_i} \right)\alpha_\pi + e^{-\tau_i} \widehat{c}_{i}  = c_i + \frac{\alpha_\pi}{2}, \quad i = 1, 2, 3, \dots 
\end{align}
where $\widehat{c}_{{i}}$ and $c_i$ satisfy the recursion for $i= 1, 2, 3, \dots$ 
\begin{align}
    \label{eq:chatfinalrecurs}
    \widehat{c}_i &=  b \tan \left(\tan^{-1}\left( \tfrac{c_{i-1}}{b} \right) - 2b \tau_i  \right) \\
    \label{eq:omegafinalrecurs}
    c_{i} &= (1-e^{-\tau_{i}})\frac{\alpha_\pi}{2} + e^{-\tau_{i}}\widehat{c}_{i} \\
    \label{eq:chatfinalrecurs0cond}
     c_0 &= \alpha_d + \frac{\delta_0}{2} \alpha_\pi.
\end{align}

Similarly to the limiting analysis in the proof of Lemma \ref{lem:Wlogconc}, we have that \eqref{eq:chatfinalrecurs}-\eqref{eq:omegafinalrecurs} corresponds to a splitting scheme of the ODE
\begin{align}
    \label{eq:codeexactwfr}
    \dot{c}_t = -2c_t^2  -c_t -2b^2 + \frac{\alpha_\pi}{2},
\end{align}
whereby for the $i$th iteration, the ODE $\dot{c}_t = -2c_t^2  -2b^2$, initialised at $c_{{i-1}}$ is solved over time $\tau_i$ yielding $\widehat{c}_{i}$ and then the ODE $  \dot{c}_t =   -c_t + \frac{\alpha_\pi}{2}$ initialised at $\widehat{c}_{i}$ is solved over time $\tau_i$ yielding $c_{i}$.  
 Finally, due to \eqref{eq:alphafinalrecurs}, $\alpha_t = c_t + \frac{\alpha_\pi}{2}$, is the log-concavity constant of $\mu_t$, the solution of the exact WFR PDE \eqref{eq:WFRpde} at time $t$, when $c_t > 0$. 
 
We now make precise the conditions on $\tau_i$ such that \eqref{eq:codeexactwfr} has a strictly positive solution for any $t > 0$.  Recall that $\tau_i$ must be chosen such that $\widehat{c}_i > 0$, which using \eqref{eq:tauthreshc0}, is possible if $0< \tau_i < \tau_i^\ast$ where 
\begin{align}
\label{eq:tauicond}
    \tau_i^\ast & = \frac{1}{2b} \tan^{-1} \left( \frac{c_{i-1}}{b} \right), \quad i = 1, 2, 3, \dots  
\end{align}
Notice that there exists a $\tau_1^\ast > 0$ whenever $\alpha_d, \delta_0 > 0$, which implies $\widehat{c}_1 > 0$ and $c_1 > \left(1 - e^{-\tau_{1}}  \right) \frac{\alpha_\pi}{2}$ and $\tau_2^\ast >  \frac{1}{2b} \tan^{-1} \left(  \frac{(1 - e^{-\tau_{1}})\alpha_\pi}{2b}  \right) > 0$.  This then implies $\widehat{c}_2 > 0$ and once again $c_2 > \left(1 - e^{-\tau_{2}} \right) \frac{\alpha_\pi}{2} $ and $\tau_3^\ast >  \frac{1}{2b} \tan^{-1} \left(  \frac{(1 - e^{-\tau_{2}})\alpha_\pi}{2b}  \right) > 0$. By induction,  $\tau_i^\ast >  \frac{1}{2b} \tan^{-1} \left(  \frac{(1 - e^{-\tau_{i-1}})\alpha_\pi}{2b}  \right) > 0$ and since $\tau_{i-1} < \tau_{i-1}^\ast$, we can instead consider the recursion 
\begin{align}
    \label{eq:taustarrecursion}
    \tau^\ast_i =  \frac{1}{2b} \tan^{-1} \left(  \frac{(1 - e^{-\tau_{i-1}^\ast})\alpha_\pi}{2b}  \right), \enskip i = 1, 2, 3, \dots  
\end{align}
It is not difficult to see that $\tau_i^\ast > 0$ for all $i = 1, 2, 3, \dots$.  Its limit, $\tau_\infty^\ast$ can be found from the fixed point equation $\tau^\ast_\infty =  g(\tau_\infty^\ast)$ where $g(\tau) := \frac{1}{2b} \tan^{-1} \left(  \frac{(1 - e^{-\tau})\alpha_\pi}{2b}  \right)$, for which $\tau_\infty^\ast = 0$ is a valid (but not unique) solution.  Now $g(\tau)$ is a Lipschitz function since $0 < \frac{dg(\tau)}{d\tau} = \frac{\alpha_\pi}{4b^2} \frac{e^{-\tau}}{1 + \left( \frac{(1 - e^{-\tau})\alpha_\pi}{2b} \right)^2} \leq \frac{\alpha_\pi}{4b^2}$ for all $\tau \geq 0$.  To see why condition \eqref{ass:b2} is necessary, consider the case $\frac{\alpha_\pi}{4b^2} < 1$.  Then by Lipschitz continuity of $g$,  $\tau_{i+1}^\ast = |g(\tau_i^\ast) - g(0)| \leq \frac{\alpha_\pi}{4b^2}|\tau_i^\ast - 0|$ since $g(0) = 0$ and $\tau_i^\ast > 0$ for all $i$.  By induction, we have that $\tau_{i}^\ast \leq \left(\frac{\alpha_\pi}{4b^2} \right)^i \tau_0^\ast$.  Since $\sum_{i=1}^\infty \left(\frac{\alpha_\pi}{4b^2} \right)^i = \frac{1}{1 - \frac{\alpha_\pi}{4b^2}}$  whenever $\frac{\alpha_\pi}{4b^2} < 1$, it holds that $\sum_{i=1}^\infty \tau_i < \sum_{i=1}^\infty \tau_i^\ast \leq \frac{1}{1 - \frac{\alpha_\pi}{4b^2}}$, so that it is impossible to construct a sequence of thresholds $\{\tau_i^\ast \}$ spanning an infinite time horizon. To see why condition \eqref{ass:b2} is sufficient for an infinite time control, we show that there exists a $\tau_\infty^\ast > 0$ whenever $\tau_1 \neq 0$.  First note that $\frac{dg(0)}{d\tau} = \frac{\alpha_\pi}{4b^2} > 0$, meaning $\tau = 0$ is an unstable fixed point.  Consider $h(\tau):= g(\tau) - \tau$ and note that $h(\tau) \rightarrow -\infty$ as $\tau \rightarrow \infty$ since $g(\tau) < \frac{1}{2b} \tan^{-1}\left( \frac{\alpha_\pi}{2b}\right) $  for all $\tau \geq  0$.  Additionally, $h(0) = 0$ and $\frac{dh(0)}{d\tau} = \frac{\alpha_\pi}{4b^2} - 1 > 0$ under condition \eqref{ass:b2} and $\frac{dh(\tau)}{d \tau} \rightarrow -1$ as $\tau \rightarrow \infty$ monotonically.  Then by continuity of $h$ and $g$, there must exist a $\tau > 0$ such that $h(\tau) = 0$, i.e. there exists a $\tau_\infty^\ast > 0$.   That is, we may construct a sequence of step sizes $\{ \tau_i\}_{i=1, 2, 3, }$ with $0< \tau_i < \min(\tau_1^\ast, \tau_\infty^\ast \neq 0)$ for all $i = 1, 2, 3, \dots$ as a discretisation of the interval $[0, T]$ for all $T> 0$, for which the recursion \eqref{eq:chatfinalrecurs}-\eqref{eq:omegafinalrecurs}  converges to the continuous time ODE \eqref{eq:codeexactwfr} as $\max_i \tau_i \rightarrow 0$.  

\vspace{0.5cm}
For the case $\alpha_h - L_\pi > 0$, following a similar line of reasoning yields 
\begin{align*}
    \dot{c}_t = -2c_t^2 - c_t + 2b^2 + \frac{\alpha_\pi}{2}. 
\end{align*}
Since this case corresponds to a uniform in time preservation of log-concavity under the W flow, there are no further restrictions when considering the WFR flow. 

Finally, we have that for both $\alpha_h - L_\pi < 0$ and $\alpha_h - L_\pi > 0$, we must solve  
\begin{align}
\label{eq:ctode}
    \dot{c}_t &= -2c_t^2 - c_t + r 
\end{align}
where $r > 0$ in both cases.  Using the change of variable $y_t = c_t + \frac{1}{4}$ and solving \eqref{eq:ctode} via separation of variables, along with the fact that $\alpha_t = \frac{\alpha_\pi}{2} + c_t$ yields \eqref{eq:logconcuniform}.


\bibliographystyle{plainnat}
\bibliography{bibrefs}

\end{document}